\documentclass{article}
\usepackage{amsmath,amssymb,amsfonts,amsthm}
\allowdisplaybreaks
\usepackage{paralist}
\usepackage{booktabs}
\usepackage[round]{natbib}
\usepackage{graphicx}
\usepackage{url}

\title{Explaining black box decisions by Shapley cohort refinement}
\author{
   Masayoshi Mase\\Hitachi, Ltd.
   \and
   Art B. Owen \\ Stanford University
   \and
   Benjamin Seiler \\ Stanford University
}

\date{September 2020}

\renewcommand{\le}{\leqslant}
\renewcommand{\ge}{\geqslant}
\renewcommand{\emptyset}{{\varnothing}}

\newtheorem{lemma}{Lemma}
\newtheorem{theorem}{Theorem}

\newcommand{\giv}{\!\mid\!}

\newcommand{\real}{\mathbb{R}}

\newcommand{\bsc}{\boldsymbol{c}}
\newcommand{\bse}{\boldsymbol{e}}
\newcommand{\bsx}{\boldsymbol{x}}
\newcommand{\bsy}{\boldsymbol{y}}
\newcommand{\bsz}{\boldsymbol{z}}
\newcommand{\bsX}{\boldsymbol{X}}

\newcommand{\val}{\mathrm{val}}

\newcommand{\bszero}{\boldsymbol{0}}
\newcommand{\bsone}{\boldsymbol{1}}

\newcommand{\e}{\mathbb{E}}
\newcommand{\var}{\mathrm{var}}
\newcommand{\rd}{\,\mathrm{d}}

\newcommand{\cs}{\mathrm{CS}}
\newcommand{\bs}{\mathrm{BS}}
\newcommand{\ces}{\mathrm{CES}}
\newcommand{\abs}{\mathrm{ABS}}

\begin{document}
\maketitle

\begin{abstract}
We introduce a variable importance measure to
quantify the impact of individual input variables to a black box function. Our measure is
based on the Shapley value from cooperative game
theory.  Many measures of variable importance
operate by changing some predictor values with
others held fixed, potentially creating unlikely or even logically impossible combinations.
Our cohort Shapley measure uses only
observed data points.
Instead of changing the value of a predictor we
include or exclude subjects similar to the target
subject on that predictor to form a similarity
cohort. Then we apply Shapley value to the cohort averages.
We connect variable importance measures from explainable AI to function decompositions from global sensitivity analysis.
We introduce a squared cohort Shapley value that splits previously studied Shapley effects over subjects, consistent with a Shapley axiom.
\end{abstract}

\section{Introduction}

Black box prediction models used in statistics, machine learning and
artificial intelligence have been able to make increasingly accurate
predictions, but it remains hard to understand those predictions.
Part of understanding predictions is understanding which
variables are important.
Variable importance in machine learning and AI has been studied recently in
\citet{stru:kono:2010,stru:kono:2014}, \citet{ribe:etal:2016},
\citet{lund:lee:2017},
\citet{sund:najm:2019},
\citet{kuma:2020}
as well as the book of \citet{moln:2018}.

A variable could be important because changing
it makes a causal difference, or because
changing it makes a large change to our predictions
or because leaving it out of a model reduces
that model's prediction accuracy
\citep{jian:owen:2003}. Importance by one of these
criteria need not imply importance by another.

In order to explain a decision like somebody being admitted
to the intensive care unit or turned down for a loan,
we are not focused on causality. The variables that
are predictive of somebody's outcome
ideally should have been the ones that were important
to the decision, but they may not have been.
Similarly, variables whose inclusion
is necessary for a good model fit are not necessarily
the ones that best explain an outcome.
The explanation should use the model that was deployed, and not any others.

We are left with a problem of explaining
the output of one specific fitted function $f(\cdot)$
applied to a vector $\bsx$ of input variables,
hereafter predictors.
For a target subject $t$ with a $d$-dimensional set of predictors $\bsx_t=(x_{t1},\dots,x_{td})$
we can change some subset of values of $\bsx_t$,
perhaps to some baseline levels in $\bsx_b$.
Then the variables that best explain why $f(\bsx_t)$
is different from $f(\bsx_b)$ are deemed important
for the decision on the target subject.
Given $2^d-1$ possible changes, there are strong
reasons to favor Shapley value \cite{shap:1952}
to combine them into an importance measure
as \citet{sund:najm:2019} and others that they survey do.
We note however, that concerns about using Shapley
value have been raised by \citet{kuma:2020}.

Mixing and matching the components
of $\bsx_t$ and $\bsx_b$
presents some problems.
The variables $x_{ij}$ and $x_{ik}$ may
show a strong correlation over subjects $i=1,\dots,n$.
Putting $x_{tj}$ and $x_{bk}$ into a
single hypothetical
point may produce an input combination
far from any that has ever been seen.
Beyond being unusual, some combinations are
physically or even logically impossible.
The changes might produce a hybrid data point
representing a patient whose systolic blood
pressure is lower than their diastolic blood
pressure.  Somebody's birth date could
follow their graduation date.  When hospital records show minimum, maximum and average levels
of blood oxygen, the hybrid point could have
mean $O_2$ below minimum $O_2$, or it could
have minimum and maximum that differ along
with a variable saying they were only measured
once (or never).
There could be important reasons to understand
effects of longitude and latitude separately, but some
combinations will not make sense, perhaps
by placing a dwelling within a body of water.

When the function $f(\cdot)$ that made
a decision was trained, it
would have seen few if any impossible
or extremely unlikely inputs.
As a result, its predictions there cannot
have been regularized suitably. Investigators
should be able to choose an importance measure that does not rely on any such values.

In this paper we work under the constraint
that only possible values can be used.
It is hard to define possible values and so
operationally we only use actually observed values.
For each predictor, every subject in the
data set is either similar to the target subject
or not similar.  Ways to define similarity
are discussed below.
Given $d$ predictors, there are $2^d$ different sets of
predictors on which subjects can be similar to the target.
We form $2^d$ different cohorts of subjects, each consisting
of subjects similar to the target on a subset of predictors,
without regard to whether they are also similar on any
of the other predictors.
At one extreme is a set of all predictors, and a cohort that
is similar to the target in every way.  At the other extreme,
the empty predictor set yields the set of all subjects. As usual, for large $d$, Monte Carlo sampling is used to approximate the Shapley value.

We can refine the grand cohort of all subjects towards the
target subject by removing subjects that mismatch the
target on one or more predictors.  The predictors that
change the cohort mean the most when we restrict to
similar subjects, are the ones that
we take to be the most important in explaining why
the target subject's prediction is different from that of the other subjects.

We call this approach `cohort refinement Shapley'
or cohort Shapley for short.
With baseline Shapley, taking account of variable
$j$ is done by replacing $x_{bj}$ by $x_{tj}$.
Cohort Shapley doesn't change the value. It refines
what we know about $\bsx_t$. Variables whose
knowledge moves the cohort mean the most
are then the important ones.

This knowledge approach has some consequences.
If one is studying redlining by an algorithm
that simply does not use a protected variable then
baseline Shapley can only give that variable zero
importance.  Cohort Shapley can give that variable
a nonzero importance measure even if it was not
in the data set at the time the model was trained.
It is thus useful for the problem of
auditing an algorithm for indirect inferences,
a task described by \citet{adle:2018}.
\citet{kuma:2020} describe this issue as
a catch-22.  In their terms, cohort Shapley
is a conditional method (changing knowledge)
while baseline Shapley is an interventional method (changing values).
The conditional approach has the problem of requiring
a joint distribution for predictors, while the interventional approach
can depend on impossible predictor combinations.
The word `importance'
does not have a unique interpretation. Multiple
ways to quantify it estimate different quantities and  involve tradeoffs. We expect that multiple methods
must be considered in practice.

\citet{sund:najm:2019} and \citet{kuma:2020}
mention a second problem with baseline Shapley.
If two predictors are highly correlated
then the algorithm might have used them
very unequally, perhaps due to an $L_1$
regularization.  It then becomes unpredictable
which will be deemed important.
The situation is more stable with cohort Shapley.
In the extreme case of identical variables
and identical similarity notions, cohort Shapley
will give them equal importance.


As mentioned above,
we define the impact of a variable through the Shapley value.
Shapley value has been used in model
explanation for machine learning \citep{stru:kono:2010,stru:kono:2014,lund:lee:2017,sund:najm:2019} and
for computer experiments \citep{owen:2014,song:nels:stau:2016,owen:prie:2017}.
KernelSHAP \citep{lund:lee:2017} computes an approximation of baseline Shapley replacing $f$ by a kernel approximation.
KernelSHAP provides an option to compute the average of Shapley values for multiple baseline vectors, and one very reasonable choice is
to use every training point as a baseline point. We call that `all baseline Shapley' (ABS), below.

Our contribution is to further
investigate the differences between baseline and cohort Shapley 
and to connect them to the global sensitivity analysis literature
which emphasizes Sobol' indices
\citep{sobo:1993}.
For surveys, see \citet{salt:ratt:andr:camp:cari:gate:sais:tara:2008}
and \citet{gamb:jano:klei:lagn:prie:2016}.
We show how both baseline Shapley value and cohort Shapley value are derived from functional decompositions.
Variable importance for a single target subject $t$ is `local' as distinct from `global' prediction importance, taken over all subjects. We consider aggregation and disaggregation
between local and global problems. While Shapley value is additive
over games, some unwanted cancellations can arise
during aggregation.
For instance, the local importances can sum to a global importance of zero.
We resolve this by
using squared differences for local and global methods.
We illustrate these measures on the well known Boston housing and Titanic data sets.


\section{Notation and Background}\label{sec:background}
The predictor vector for subject $i$ is
$\bsx_i =(x_{i1},\dots,x_{id})$ where $d$ is the number of predictors
in the model.  Each $x_{ij}$ belongs to a set $X_j$ which
may consist of real or binary variables or some
other types.  There is a black box function $f:\bsx\to\real$ that is used
to predict an outcome for a subject with predictor vector $\bsx$.  We write
$y=f(\bsx)$ and $y_i=f(\bsx_i)$.
For local explanation we would like to quantify the importance of predictors $x_{tj}$ on $y_t = f(\bsx_t)$.
We assume that $t$ is in the
set of subjects with available data, although this subject might not have been
used in training the model.
A baseline point $\bsx_b$ might not be one of the subjects.  For instance if $\bsx_b=(1/n)\sum_{i=1}^n\bsx_i$ it will ordinarily not be a data point. For example, $\bsx_{b,j}$ might be in $(0,1)$ for a binary variable $j$.

The set $\{1,2,\dots,n\}$ of subjects is denoted by $1{:}n$ and the set of predictors is $1{:}d$.
We will need to manipulate subsets of $1{:}d$.
For $u\subseteq 1{:}d$ we let $|u|$ be its
cardinality.  The complementary set $1{:}d\setminus u$
is denoted by $-u$, especially in subscripts.
Sometimes a point must be created by combining
parts of two other points.
The point $\bsy =\bsx_u{:}\bsz_{-u}$ has $y_j =x_j$ for $j\in u$
and $y_j = z_j$ for $j\not\in u$.
Furthermore, we sometimes use $j$ and $-j$ in place of
the more cumbersome $\{j\}$ and $-\{j\}$.
For instance, $\bsx_{t,-j}{:}\bsx_{b,j}$ is what we get by changing the $j$'th predictor of the target subject to the baseline level.
We also use $u + j$ for $u\cup\{j\}$.

\subsection{Function Decompositions}
Function decompositions, also called
high dimensional model representations (HDMR),
write a function of $d$ inputs as a sum of functions,
each of which depend only on one of the $2^d$
subsets of inputs.
Because $f$ and $\bsx$ have other uses in this paper,
we present the decomposition for $g(\bsz)$.
Let $g$ be a function of $\bsz=(z_1,\dots,z_d)$
with $z_j\in Z_j$.
In these decompositions we write
$$
g(\bsz) = \sum_{u\subseteq1{:}d}g_u(\bsz)
$$
where $g_u(\bsz)$ depends on $\bsz$ only through $\bsz_u$.
Many such decompositions are possible~\citep{kuo:sloa:wasi:wozn:2010}.

The best known decomposition for this problem is the functional
analysis of variance (ANOVA) decomposition
\citep{hoef:1948,sobo:1969}.
It applies to random $\bsz$ with
independent components $z_j\in Z_j$.
If $\e( g(\bsz)^2)<\infty$,
then we write
\begin{align*}
g_\emptyset(\bsz) &= \mu \equiv \e(g(\bsz)),\quad\text{followed by}\\
g_u(\bsz)
&= \e\biggl(g(\bsz) - \sum_{v\subsetneq u}g_v(\bsz)\Bigm| \bsz_u\biggr)\\
&= \e(g(\bsz)\giv \bsz_u) - \sum_{v\subsetneq u}g_v(\bsz),
\end{align*}
for non-empty $u\subseteq1{:}d$.
The effects $g_u$ are orthogonal in that for subsets
$u\ne v$, we have
$\e( g_u(\bsz)g_v(\bsz))=0.$
Letting $\sigma^2_u = \var(g_u(\bsz))$, it follows from this orthogonality
that
$$
\sigma^2(g)\equiv \var(g(\bsz)) = \sum_{u\subseteq1{:}d}\sigma^2_u(g).
$$
We can recover effects from conditional expectations, via inclusion-exclusion,
\begin{align}\label{eq:incexcanova}
g_{u}(\bsz) &= \sum_{v\subseteq u}(-1)^{|u-v|}\e( g(\bsz) \giv \bsz_v).
\end{align}
See \citet{owen:2013} for history and derivations of this
functional ANOVA.

We will need the anchored decomposition, which goes back
at least to \citet{sobo:1969}.
It is also called cut-HDMR \citep{alis:rabi:2001} in chemistry,
and finite differences-HDMR in global sensitivity analysis \citep{sobo:2003}.
We begin by picking a reference point $\bsc$
called the anchor, with $c_j\in Z_j$ for $j=1,\dots,d$.
The anchored decomposition is
\begin{align*}
g(\bsz)  & = \sum_{u\subseteq1{:}d}g_{u,\bsc}(\bsz),\quad\text{with}\\
g_{\emptyset,\bsc}(\bsz) & = g(\bsc),\quad\text{and}\\
g_{u,\bsc}(\bsz) & = g(\bsz_u{:}\bsc_{-u})
-\sum_{v\subsetneq u}g_{v,\bsc}(\bsz).
\end{align*}
We have replaced averaging over $\bsz_{-u}$ by plugging
in the anchor value via $\bsz_{-u}=\bsc_{-u}$.
If $j\in u$ and $z_j=c_j$, then $g(\bsz)_{u,\bsc}=0$.
We do not need independence of the $z_j$, or even randomness for them
and we do not need mean squares.
What we need is
that when $g(\bsz)$ is defined, so is $g(\bsz_u{:}\bsc_{-u})$
for any $u\subseteq1{:}d$.

The main effect in an anchored decomposition is
$g_{j,\bsc}(\bsz) = g(\bsz_j{:}\bsc_{-j})-g(\bsc)$
and the two factor term for indices $j\ne k$ is
\begin{align*}
g_{\{j,k\},\bsc}(\bsz)
 = g(\bsz_{\{j,k\}}{:}\bsc_{-\{j,k\}})
-g_{j,\bsc}(\bsz) -g_{k,\bsc}(\bsz)-g_\emptyset(\bsz)\\
= g(\bsz_{\{j,k\}}{:}\bsc_{-\{j,k\}})
-g(\bsz_j{:}c_{-j})
-g(\bsz_k{:}c_{-k})
+g(\bsc).
\end{align*}
The version of~\eqref{eq:incexcanova} for the anchored decomposition is
\begin{align*}
g_{u,\bsc}(\bsz) &= \sum_{v\subseteq u}(-1)^{|u-v|}g(\bsz_v{:}\bsc_{-v}),
\end{align*}
as shown by \citet{kuo:sloa:wasi:wozn:2010}.

\subsection{Shapley Value}\label{sec:shapleyvalue}

Shapley value \citep{shap:1952} is used in game theory
to define a fair allocation of rewards to a team that
has cooperated to produce something of value.
Suppose that a team of $d$ people produce
a value $\val( 1{:}d)$, and that we have at our disposal
the value $\val(u)$ that would have been produced
by the team $u$,
for all $2^d$ teams $u\subseteq1{:}d$.
Let $\phi_j$ be the reward for player $j$.

Shapley introduced quite reasonable criteria:
\begin{compactenum}[\quad 1)]
\item Efficiency: $\sum_{j=1}^d\phi_j = \val(1{:}d)$.
\item Symmetry: If $\val(u+i)=\val(u+j)$ for all $u\subseteq1{:}d\setminus\{i,j\}$,
then $\phi_i=\phi_j$.
\item Dummy: if $\val(u+j) = \val(u)$ for all $u\subseteq1{:}d\setminus\{j\}$,
then $\phi_j=0$.
\item Additivity:
if $\val(u)$ and $\val'(u)$ lead to values $\phi_j$ and $\phi_j'$
then the game producing $(\val+\val')(u)$ has values $\phi_j+\phi'_j$.
\end{compactenum}
He found that the unique valuation that satisfies all
four of these criteria is
\begin{align}\label{eq:shapj}
\phi_j = \frac1d\sum_{u\subseteq -j}
{d-1\choose |u|}^{-1}(\val(u+j)-\val(u)).
\end{align}

Formula~\eqref{eq:shapj} is not very intuitive.
Another way to explain Shapley value is as follows.
We could build a team from $\emptyset$
to $1{:}d$ in $d$ steps, adding one member at a time.
There are $d!$ different orders in which
to add team members. The Shapley value $\phi_j$
is the increase in value coming from the addition of
member $j$, averaged over all $d!$ different orders. For large $d$, random permutations can be sampled to estimate~$\phi_j$.
From equation~\eqref{eq:shapj} we see that Shapley value does not change if we add or  subtract the same quantity from all $\val(u)$. It can be convenient to make $\val(\emptyset)=0$.

\subsection{Shapley Value for Local Explanation}
When we apply Shapley value to the black box function $f(\bsx)$ of the
target data subject $\bsx_t$, we define a player as an input predictor
$x_{tj}$.
The presence of the player $j$ means that the value of
$x_{tj}$ is known to the black box function $f(\bsx_t)$.
Shapley additive explanation (SHAP) of \citet{lund:lee:2017} uses the value function
$$
\val(u) = \e(f(\bsx)\giv\bsx_{u}) - \e(f(\bsx)).
$$
The value of $\e(f(\bsx)\giv\bsx_u)$ depends on
the data distribution that is implicitly or explicitly assumed.
\citet{sund:najm:2019} and \citet{kuma:2020} give various definitions.

Interventional Shapley values assume that all combinations of input predictors are possible. That is, the inputs are functionally independent. 
Baseline Shapley (BS) \citep{sund:najm:2019} defines the value function
for a predictor set as
$$
\val_{\bs}(u) = f(\bsx_{t,u}{:}\bsx_{b,-u}) - f(\bsx_b).
$$
It replaces predictors one by one from baseline to subject value.
KernelSHAP \citep{lund:lee:2017} computes baseline Shapley value for a kernel regression approximation to $f$, for computational efficiency. It also supports ABS, which like the quantitative input influence (QII) measure of \citet{datt:2016} assumes statistically independent predictors, using the empirical margins.

Conditional Shapley values such as cohort Shapley allow
for predictor variable dependence.
When the similarity measures are all identity, then cohort Shapley, defined below, is equivalent to the
Conditional Expectation Shapley of \citet{sund:najm:2019} on the empirical distribution $\hat D$ of the data, denoted
$\ces(\hat D)$.
%
Cohort Shapley extends $\ces(\hat D)$ with similarity measures
to better handle continuous variables.
CondKernelSHAP \citep{aas:2019}  implements conditioning using kernel weights derived from the Mahalanobis distance and optionally restricted to $K$-nearest neighbor points only. 

\section{Cohort Shapley}\label{sec:simcohorts}
%
%
%
%
The cohort Shapley measure designates subjects as similar or not similar to the target subject $t$ for each of $d$ predictors. Then for each subset of predictors, there is a set of subjects similar to the target for all of those predictors. We call these subsets cohorts.  The cohorts are not empty because they all include subject $t$. 
Cohort Shapley value is Shapley value
applied to cohort means.
First we describe similarity.

\subsection{Similarity}
For each predictor $j$, we define a target-specific
similarity function $z_{tj}:X_j\to\{0,1\}$.
If $z_{tj}(x_{ij})=1$, then subject $i$ is considered
to be similar to subject $t$ as measured by predictor $j$.
Otherwise $z_{tj}(x_{ij})=0$ means that subject $i$
is dissimilar to subject $t$ for predictor $j$. We always have $z_{tj}(x_{tj})=1$.
The simplest similarity is identity:
$$
z_{tj}(x_{ij}) =
\begin{cases}
1, & x_{ij}=x_{tj},\\
0, & \mathrm{else},
\end{cases}
$$
which is reasonable for binary predictors or those
taking a small number of levels.
For real-valued predictors, there may be no $i\ne t$
with $x_{ij}=x_{tj}$ and then  we might instead take
$$
z_{tj}(x_{ij}) =
\begin{cases}
1, & |x_{ij}-x_{tj}|\le\delta_{tj},\\
0, & \mathrm{else},
\end{cases}
$$
where subject matter experts have chosen $\delta_{tj}$.
Taking $\delta_{tj}=0$ recovers the identity measure of similarity.
The two similarity measures above generate an
equivalence relation on $X_j$, if $\delta_{tj}$
does not depend on $t$. In general, we do not
need $z_{tj}$ to be an equivalence.
For instance, we do not need $z_{tj}(x_{ij})=z_{ij}(x_{tj})$
and would not necessarily have that if we used relative
distance to define similarity, via
$$
z_{tj}(x_{ij}) =
\begin{cases}
1, & |x_{ij}-x_{tj}|
\le\delta_{j}|x_{tj}|,\\
0, & \mathrm{else}.
\end{cases}
$$
In some settings equivalence is a useful simplification. 

\subsection{Cohorts of $t$}

We use $1{:}n=\{1,2,\dots,n\}$ to define our set of
subjects. Let
$$C_{t,u}
=
\{ i \in 1{:}n \mid z_{tj}(x_{ij})=1,\,\ \text{for all $j\in u$}
\},$$
with $C_{t,\emptyset}=1{:}n$ by convention.
Then $C_{t,u}$ is the cohort of subjects that are similar
to the target subject for all predictors $j\in u$
but not necessarily similar for any predictors $j\not\in u$.
These cohorts are never empty, because we always have $t\in C_{t,u}$.
We write $|C_{t,u}|$ for the cardinality of the cohort.
As the cardinality of $u$ increases, the cohort
$C_{t,u}$ focuses in on the target subject.

Given a set of cohorts, we define cohort averages
$$
\bar y_{t,u} = \frac1{|C_{t,u}|}\sum_{i\in C_{t,u}}y_i.
$$
Then the value of set $u$ is
$$
\val_{\cs}(u) =\bar y_{t,u}-\bar y_{t,\emptyset}=\bar y_{t,u}-\bar y,
$$
where $\bar y=(1/n)\sum_{i=1}^ny_i$.
The last equality follows because the cohort with $u=\emptyset$ is the whole data set.
The total value to be explained is
$$\val_{\cs}(1{:}d) =\bar y_{t,1:d}-\bar y_{t,\emptyset}
=\bar y_{t,1{:}d}-\bar y.
$$
It may happen that $C_{t,1:d}$ is the singleton $\{t\}$.
In that case, the total value to be explained is $f(\bsx_t)-\bar y$.
In this and other settings, some of the value $\bar y_{t,u}$ may be the
average of a very small number of subjects' predictions,
and potentially poorly determined but much better
determined than points that are extreme extrapolations
or are even impossible.

\subsection{Importance Measures}

Equation~\eqref{eq:shapj} yields cohort Shapley values
$$
\phi_j^{\cs} = \frac1d \sum_{u\subseteq -j}
{ d-1 \choose |u|}^{-1} ( \bar y_{t,u+j}-\bar y_{t,u}).
$$
Here $\bar y_{t,u+j}-\bar y_{t,u}$ is the difference that refining
on predictor $j$ makes when we have already refined
on the predictor set~$u$.

\section{Disaggregation of Global Importance}\label{sec:global}

We use Shapley value for local explanation because it uses no impossible data combinations.
%
Here we show that the Shapley effects of \citet{song:nels:stau:2016} can be disaggregated into a local squared cohort Shapley effect, when the predictors are given their empirical distribution.

Shapley effects arise from a value function $\val_\var(u)=\var(\e(y\giv\bsx_u))$.
When $\bsx$ has independent components, then \citet{owen:2014} shows that the Shapley value for $j$ is
\begin{align}\label{eq:shapvar}
\phi_j^{\var} = \sum_{u\subseteq1{:}d, \,j\in u}\frac{\sigma^2_u}{|u|},
\end{align}
but the dependent data case is much harder.

We may write the value function as
\begin{align*}
\val_{\var}(u) 
&= \int(\e(f({\bsx})\giv\bsx_u) - \e(f(\bsx)))^2 P(\bsx_u)\rd{\bsx_u}
\end{align*}
where $P(\bsx)$ is the probability density function of $\bsx$
and we retain the same notation for discrete $\bsx$.
We take $P(\bsx)$ to be the empirical distribution of the data set $\bsX=\{\bsx_1,\dots,\bsx_n\}$ coarsened by a similarity function. Then $P(\bsx\giv\bsx_{t,u})$ represents 
the discrete distribution on $\bsx_i$ obtained by choosing subject $i$ uniformly at random from the cohort $C_{t,u}$. 
Then under~$P$, 
$$\e(f(\bsx)\giv\bsx_{t,u}) = \frac1{|C_{t,u}|} \sum_{i \in C_{t,u}}f(\bsx_i)=\bar y_{t,u}.$$ 
We also have $\e(f(\bsx)) = \bar y$ for this $P$.
Finally, $P$ places equal probability $1/n$ on the data and so
$$
\val_\var(u) =\frac1n\sum_{t=1}^n(\bar y_{t,u}-\bar y)^2.
$$
%


Next, in addition to ordinary cohort Shapley with value function $\val_{\cs}(t,u)=\bar y_{t,u}-\bar y_\emptyset=\bar y_{t,u}-\bar y$ we introduce
a squared version with
$$
\val_{\cs2}(t,u) = (\bar y_{t,u}-\bar y)^2.
$$
Then,
\begin{align}\label{eq:partitionvar}
\val_{\var}(u)=\frac{1}{n} \sum_{i=1}^n\val_{\cs2}(t,u).
\end{align}

From the additivity axiom of Shapley value, the variance explained Shapley value $\phi^{\var}_j$ is the average of the squared version of cohort Shapley value $\phi^{\cs2}_j$,
\begin{align}\label{eq:partitionphi}
\phi^{\var}_j({\bf X}) = \frac{1}{n}\sum_{i=1}^n \phi^{\cs2}_j(\bsx_i).
\end{align}
The point of equations~\eqref{eq:partitionvar} and \eqref{eq:partitionphi} is that squared values give an aggregation/disaggregation between local and global measures that satisfies the Shapley axioms.   Unsquared Shapley values, whether baseline or cohort, give uninteresting aggregate Shapley values of zero.

\section{Comparison of Shapley Value Importance}
Here we compare importances of three Shapley values, BS, ABS, and CS.
For BS, the baseline is $\bsx_b = (1/n)\sum_{i=1}^n\bsx_i$.
For ABS, all subjects are used for the baselines.
We also define squared versions.
Here, for simplicity, we assume no other subject matches the target on all predictors. Then the total value function of CS and its squared version are
\begin{align*}
\val_{\cs}(t,1{:}d) &= f(\bsx_t) - \bar y\quad\text{and}\\
\val_{\cs2}(t,1{:}d) &= (f(\bsx_t) - \bar y)^2.
\end{align*}

The value functions of BS and its squared difference version are
\begin{align*}
\val_{\bs}(t,1{:}d) &= f(\bsx_t) - f(\bsx_b)\quad\text{and}\\
\val_{\bs2}(t,1{:}d) &= (f(\bsx_t) - f(\bsx_b))^2.
\end{align*}
%
Now if $f(\bsx_b)=\bar y$, as for instance it would be for linear $f$ and $\bsx_b$ equal to the average of $\bsx_i$, then both $\val_\bs(t,1{:}d)=\val_\cs(t,1{:}d)$ and $\val_{\bs2}(t,1{:}d)=\val_{\cs2}(t,1{:}d)$.

The value function of ABS and its squared difference version are defined as
\begin{align*}
\val_{\abs}(t,1{:}d) &= \frac{1}{n} \sum_{i=1}^{n} (f(\bsx_t)- f(\bsx_i))
= f(\bsx_t) - \bar y\\
\val_{\abs2}(t,1{:}d) &= \frac{1}{n} \sum_{i=1}^{n} (f(\bsx_t) - f(\bsx_i))^2.
\end{align*}
We see that $\val_{\abs} = \val_{\cs}$. We can also show that
$\val_{\abs2} \ge\val_{\cs2}$
with equality when and only when $f(\bsx_b)=f(\bsx_t)$ for all $t$, a very unlikely possibility.
Thus, while ABS decomposes the same value that CS does, their squared counterparts decompose quite different quantities. 


\section{Derivation from Function Decomposition}
%
%
%
\citet{stru:kono:2010} derived a Shapley value under independent predictors uniformly distributed over finite discrete sets but they could as well be countable
or continuous and non-uniform, so long as they are independent (see the appendix).
We give a proof of this different from theirs, and show how both baseline Shapley value and cohort Shapley value are derived from functional decompositions.

\subsection{Importance Measures}
For Shapley value, every variable is either `in' or `out', and
so binary variables underlie the approach.
Here we compute Shapley values based on
function decompositions of a function $g$ defined on $\{0,1\}^d$.
The $2^d$ values of that function might themselves
be expectations, like the cohort mean in cohort Shapley
or the quantity \citet{stru:kono:2010} use (see the appendix),
but for our purposes here they are just $2^d$ numbers.


We use $\bse_j$ to represent the binary
vector of length $d$ with a one in position $j$
and zeroes elsewhere. This is the $j$'th standard
basis vector. We then generalize it to
$\bse_u=\bsone_u{:}\bszero_{-u}$
for $u\subseteq1{:}d$.
An arbitrary point in $\{0,1\}^d$ is denoted by $\bsz$.

Let $g$ be a function on $\{0,1\}^d$.
In our applications, the
total value to be explained is $g(\bsone)-g(\bszero)$,
with $\bsone$ corresponding to matching
the target in all $d$ ways and $\bszero$ corresponding
to no matches at all.
The value contributed by $u\subseteq1{:}d$
is $g(\bse_u)-g(\bszero)$.

\subsection{Shapley Value via Anchored Decomposition}

Because we use the anchored decomposition for functions
on $\{0,1\}^d$  instead of the ANOVA, we
do not need to define a distribution for $\bsz$.
The anchored decomposition on $\{0,1\}^d$
with anchor $\bsc=\bszero$
has a simple structure.
\begin{lemma}\label{lem:binanchdecomp}
For integer $d\ge1$,
let $g:\{0,1\}^d$ have the anchored decomposition
$g(\bsz) = \sum_{u\subseteq1{:}d}g_{u,\bszero}(\bsz)$
with anchor $\bszero$.
Then
$g_{u,\bszero}(\bse_w) = g_{u,\bszero}(\bsone)1_{u\subseteq w}$,
where $\bse_w=\bsone_w{:}\bszero_{-w}$.
\end{lemma}
\begin{proof}
See the appendix.
\end{proof}

Now we find the Shapley value for a function
on $\{0,1\}^d$ in an anchored decomposition.
\citet{stru:kono:2010} proved
this earlier using different methods.

\begin{theorem}\label{thm:shapanchored}
Let $g(\bsz)$ have the anchored decomposition
with terms $g_{u,\bszero}(\bsz)$ for $\bsz\in\{0,1\}^d$.
Let the set $u\subseteq 1{:}d$ contribute value
$g(\bse_u)-g(\bszero)$.
Then the total value is
$g(\bsone)-g(\bszero)$, and the Shapley
value for variable $j\in 1{:}d$ is
\begin{align}\label{eq:shapanchored}
\phi_j
= \sum_{u, \, j\in u}\frac
{g_{u,\bszero}(\bsone)-g_{u,\bszero}(\bszero)}{|u|}
= \sum_{u, \, j\in u}\frac
{g_{u,\bszero}(\bsone)}{|u|}.
\end{align}
\end{theorem}
\begin{proof}
We give a purely combinatorial proof of this result, using the hockey-stick identity, in our appendix.
\end{proof}


\subsection{Baseline, Cohort and Variance Shapley}
For $u\subseteq1{:}d$ and $\bse_u=\bsone_u{:}\bszero_{-u}\in\{0,1\}^d$ the $g$ function for baseline Shapley is
$$g(\bse_u)=\val_{\bs}(t,u) = f(\bsx_{t,u}{:}\bsx_{b,-u}) - f(\bsx_b).$$
The corresponding function for cohort Shapley is
$$g(\bse_u) = \val_{\cs}(t,u) = \bar y_{t,u}-\bar y.$$
In variance Shapley $g(\bse_u)=\var(\e(f(\bsx)\giv\bsx_u))$. It is Shapley value applied to the anchored decomposition of $g$ where the $2^d$ values $g(\bse_u)$ are in turn defined through an ANOVA decomposition.


\section{Examples}\label{sec:examples}

In this section we include some numerical examples of cohort Shapley.
First we look at variables that predict survival in the Titanic data set. Then we look at the variables that predict value in the Boston housing data set. 

\subsection{Titanic Data}\label{sec:titanic}
Here we consider a subset of the Titanic passenger data set
(from \url{http://biostat.mc.vanderbilt.edu/wiki/pub/Main/DataSets/titanic3.csv})
containing
1045 individuals with complete records.
This data was collected from Encyclopedia Titanica (\url{https://www.encyclopedia-titanica.org/}) and
has been used by Kaggle (see
\url{https://www.kaggle.com/c/titanic/data}) to illustrate machine learning.
As the function of interest, we construct a logistic regression model which predicts
`survival'
based on the predictors
`pclass', `sex', `age', `sibsp', 'parch', and `fare' (defined at the Vanderbilt web site).
Our logistic model outputs an estimated probability of survival, $f(\bsx_t)\in[0,1]$. To calculate the cohort Shapley values, we define similarity as exact for the discrete predictors `pclass', `sex',
`sibsp', and `parch'
and within a distance of less than 1/10 of the variable range on the continuous predictors `age' and `fare'.

Figure~\ref{fig:titan_vs} shows the variance Shapley values, known as `Shapley effects' \citep{song:nels:stau:2016}. They  decompose the total variance into importance of predictor variables.
The results indicate that `sex' is the most important followed by `pclass', and `fare'.

\begin{figure}[t]
  \centering
  \includegraphics[width=0.8 \textwidth]{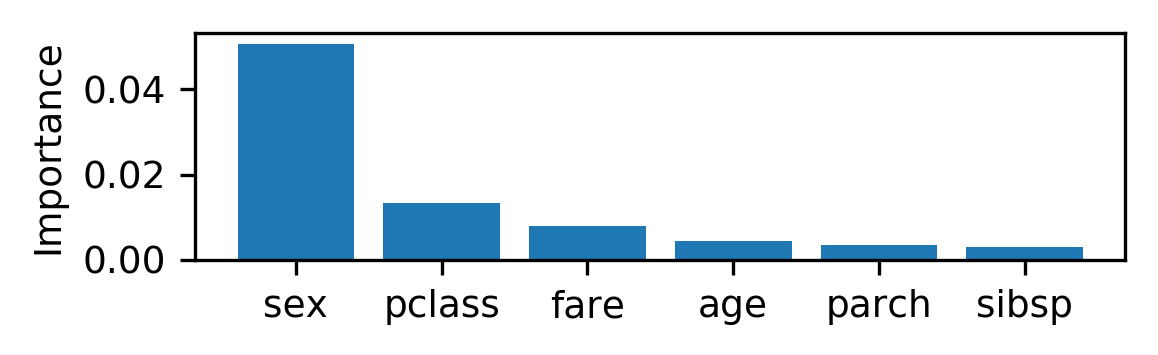}
  \caption{Shapley effects for the  Titanic data set.}
  \label{fig:titan_vs}
\end{figure}

\begin{figure}[t]
  \centering
  \includegraphics[width=0.8 \textwidth]{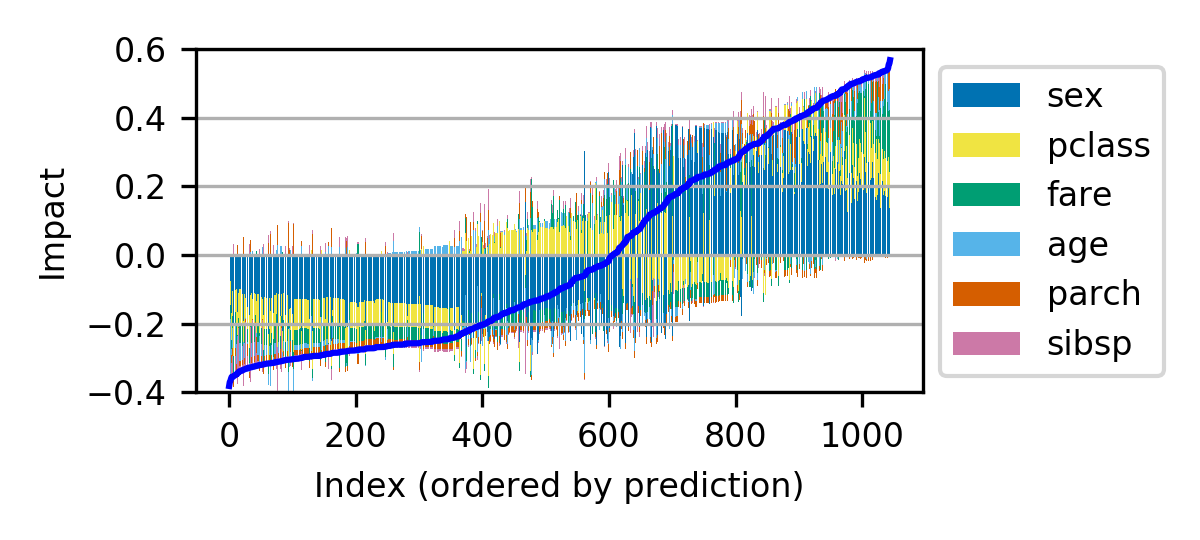}
  \caption{Cohort Shapley values stacked vertically for all passengers, ordered by estimated survival probability.  The blue overlay is $f(\bsx_t)-\bar y$ for each passenger.}
  \label{fig:titan_SHlog}
\end{figure}

Figure~\ref{fig:titan_SHlog} shows the cohort Shapley values for each predictor stacked vertically for every individual. The individuals are ordered by their predicted survival probability.  Starting at zero, we plot a blue bar up or down according to the cohort Shapley value for the `sex' variable.  Then comes a
yellow bar for `pclass' and so on.
The squared version of the cohort Shapley is a decomposition of Variance Shapley in Figure~\ref{fig:titan_vs}. For a  comparison of cohort, baseline, all baseline Shapley values, and for their squared versions, see the appendix.

A visual inspection of Figure~\ref{fig:titan_SHlog} reveals clusters of individuals with similar Shapley values for which we could potentially develop a narrative.  As just one example, we see passengers with indices between roughly
400 and 600
who have negative Shapley values for `sex' but positive Shapley values for `pclass' while their predicted value is below the mean.  Many of the these passengers are men who are not in the lowest class.

\subsection{Boston Housing Data}\label{sec:boston}
The Boston housing dataset
has 506 data points with 13 predictors and the median house value as a response \citep{harr:rubi:1978}.
Each data point corresponds to a vicinity in the Boston area.
We fit a regression model to predict the house price from
the predictors using XGBoost \citep{chen:gues:2016}.

This dataset is of interest to us because it includes some
striking examples of dependence in the predictors.
For instance, the variables `CRIM' (a measure of per capita crime)
and `ZN' (the proportion of lots zoned over 25{,}000 square feet)
can be either near zero or large, but none of the 506 data
points have both of them large and similar phenomena
can be seen in some other scatterplots. 

We will compute baseline Shapley and cohort Shapley for one target point.
That one is the 205'th case in the sklearn python library
and also in the mlbench R package. 
This target was chosen to be one for which some synthetic
points in baseline Shapley would be far from any real data, but we did not optimize
any criterion measuring that distance, and many of the other 506 points share
that property.
For Shapley values of all subjects, see the appendix.
For cohort Shapley, we consider predictor values to be similar if their distance is less than
1/10 of the difference between the 95'th and 5'th
percentiles of the predictor distribution.

Figure~\ref{fig:boston_s204_cohort_bdys}
shows two scatterplots of the Boston housing data.
It marks the target and baseline points, depicts the cohort
boundaries and it shows housing value in gray scale.
The baseline point is $\bsx_b =(1/n)\sum_{i=1}^n\bsx_i$, the
sample average, and it is not any individual subject's point partly because it
averages some integer valued predictors.
Here, the predicted house prices are
28.38 for the subject and 13.43 for the baseline.
The figure also shows some of
the synthetic points used by baseline Shapley.
Some of those points are far from any real data points even in
these two dimensional projections.
Model fits at such points are questionable.

\begin{figure}[t]
  \centering
  \includegraphics[width=0.8 \textwidth]{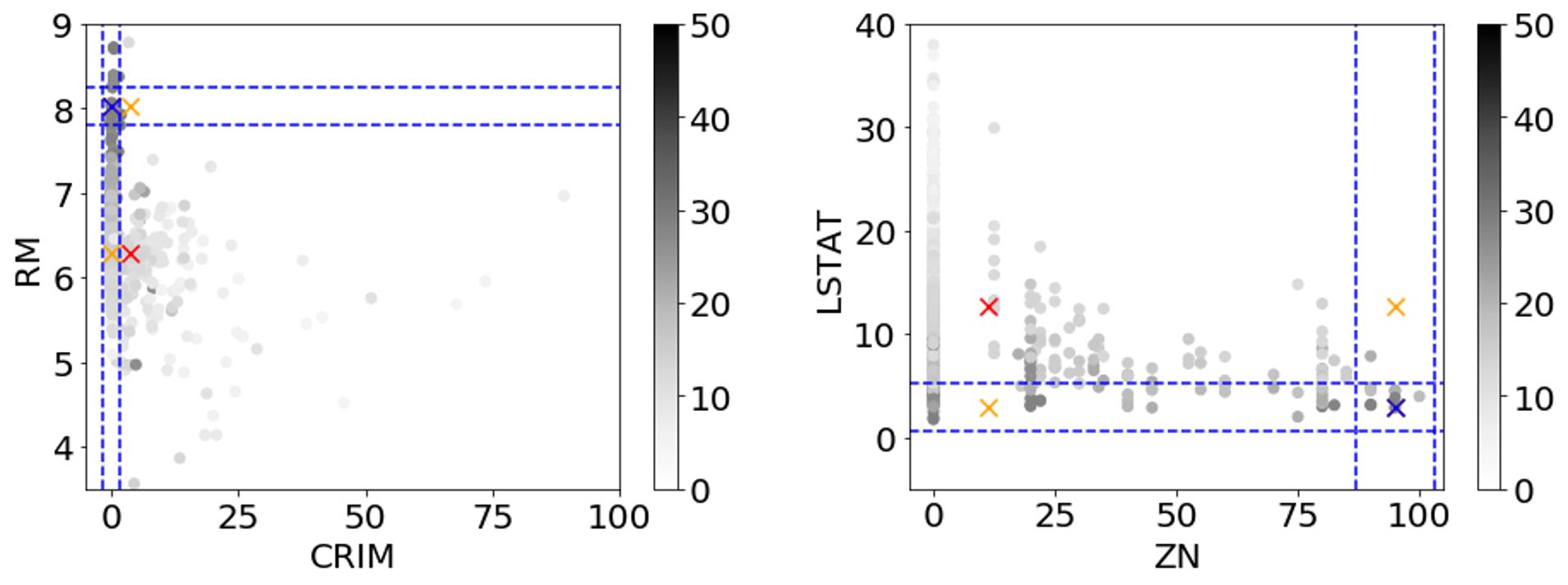}
  \caption{
Two scatterplots of the Boston housing data.
The target point is a blue X.  The baseline is a red X.
Synthetic points used by baseline Shapley are orange X's.
Dashed blue lines delineate the cohorts we used.}
  \label{fig:boston_s204_cohort_bdys}
\end{figure}

Figure~\ref{fig:boston_s204_value} shows
cohort Shapley and baseline Shapley values for this target
subject. 
On baseline Shapley, we see that `CRIM',
`RM', and `LSTAT' have very large impact
and the other variables do not.
For cohort Shapley, the most impactful predictors
are  `RM', `ZN' and `LSTAT' followed by a very
gradual decline.

 \begin{figure}[tb]
   \centering
   \includegraphics[width=0.8 \textwidth]{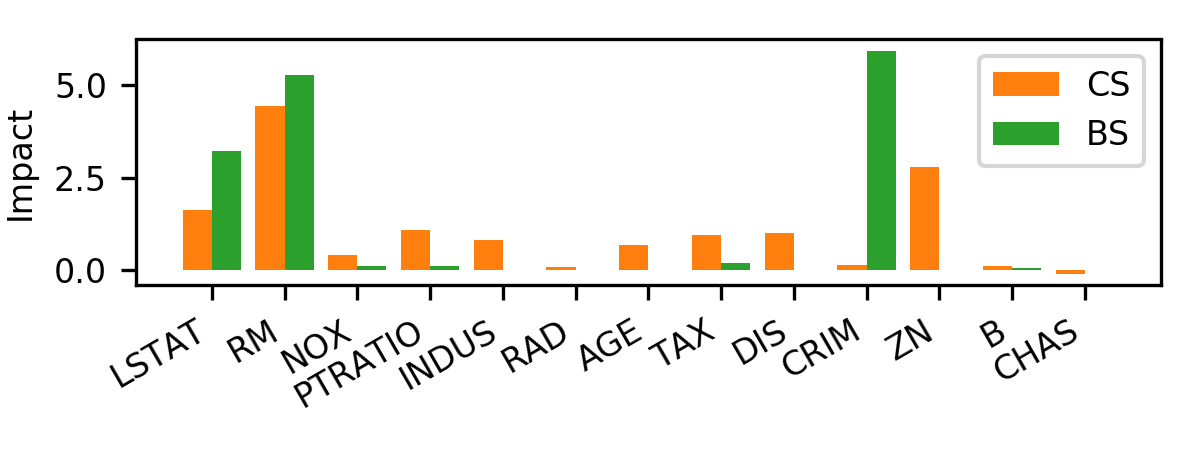}
   \caption{Cohort Shapley and baseline Shapley values for subject 205 of the Boston housing data.}
   \label{fig:boston_s204_value}
 \end{figure}

In baseline Shapley `CRIM'
was the most important variable, while in cohort Shapley
it is one of the least important variables.
We think that the explanation is from the way that baseline
Shapley uses the data values at the upper orange cross
in the left plot of Figure~\ref{fig:boston_s204_cohort_bdys}.
The predicted price for a house at the synthetic
point given by the upper orange cross is 14.17,
which is much smaller than that of the subject, and
even quite close to the baseline mean.
This leads to the impact of `CRIM' being very high.
Data like that synthetic point were not present in the
training set and so that value represents an extrapolation where
we do not expect a good prediction.
We believe that an unreliable prediction there gave
the extreme baseline Shapley value that we see for `CRIM'.

Related to the prior point,
refining the cohort on `RM' reduces its cardinality
much more than refining the cohort on `CRIM' does.
Because cohort Shapley uses averages of actual subject
values, refining the target on `CRIM' removes fewer
subjects and in this case makes a lesser change.

The right panel in Figure~\ref{fig:boston_s204_cohort_bdys}
serves to illustrate the effect of dependent predictors on cohort
Shapley value. The model for price hardly uses `ZN', if at all,
and the baseline Shapley value for it is zero.
Baseline Shapley atributes a large impact to `LSTAT' and
nearly none to `ZN'.
For either of those predictors, the cohort mean is higher
than the global average, and
both `LSTAT' and `ZN' have high impact in cohort Shapley.

We can explain the difference as follows.
As `ZN' increases, the range of `LSTAT' values
narrows, primarily by the largest `LSTAT' values decreasing
as `ZN' increases.
Refining on `ZN'  has the side effect of lowering `LSTAT'.
Even if `ZN' itself is not in the model, the cohort Shapley
value captures this effect.
Baseline Shapley cannot impute a nonzero impact
for a variable that the model does not use.

\section{Discussion}\label{sec:discussion}

Cohort Shapley resolves two conceptual problems in baseline
Shapley and other interventional methods. First, it does not use any impossible
or even unseen predictor combinations. Second, if two predictors
are identical then it is perfectly predictable that their importances
will be equal rather than subject to details of which black box model
was chosen or what random seed if any was used in fitting that model. Cohort Shapley also allows the user to study indirect influence which is impossible via baseline Shapley.

It does not escape the catch-22 that \citet{kuma:2020} mention.  Cohort Shapley requires the user to specify a distribution on predictors.  That distribution is the empirical distribution coarsened by a similarity definition.  The importance measure will depend on the chosen similarity measure which can be an advantage but may also be a burden.  
Baseline Shapley and cohort Shapley have different counterfactuals
and they address different problems.

There is a related literature on how finely a continuous
variable should be broken into categories.
\citet{gelm:park:2009} suggest as few as three levels
for the related problem of choosing a discretization
prior to fitting a model.
We have used more levels but this remains an
area for future research.

Cohort Shapley supports another choice not possible with baseline Shapley. We can replace the function values $f(\bsx_i)$ by the original response values in the cohort averages. The resulting sensitivity measures are for a similarity-based nearest neighbor rule.

In our appendix we investigate numerically the proportion of unrealistic data points created when sampling independently from marginal distributions.  Using a working definition of `realistic' we find that marginal sampling produces many more unrealistic values than one gets in hold out data sets.

\section*{Ethics Statement}
Our cohort Shapley method can be used to detect unethical practices such as redlining, even when the protected variable or variables were not used in the prediction model. We take care to point out that a large cohort Shapley value for a protected variable is not in itself proof of redlining. Followup investigations would be required to reach that conclusion.

\section*{Acknowledgments}
We thank Masashi Egi of Hitachi for valuable comments.
This work was supported by the U.S. National Science Foundation
under grant IIS-1837931.

\bibliographystyle{apalike}
\bibliography{paper}

\vfill\eject

\section*{Appendix 1 Approach of \citet{stru:kono:2010}}
To get a Shapley value for predictor variables,
we must first define the value produced by a subset of them.
The approach of  \citet{stru:kono:2010}
begins with a vector $\bsx$
of independent random predictors
from some distribution $F$.
They used independent predictors uniformly distributed
over finite discrete sets but they could as well be countable
or continuous and non-uniform, so long as they are independent.
For a target subject $t$, let
$f(\bsx_t)$ be the prediction for that subject.
They define the value of the predictor set $u\subseteq1{:}n$ by
$$
\Delta(u) = \e( f(\bsx_{t,u}{:}\bsx_{r,-u}))-\e(f(\bsx_r))
$$
with expectations taken under a
random baseline $\bsx_r\sim F$.
The distribution $F$ has independent components.
In words, $\Delta(u)$ is the expected change in
our predictions at a random point $\bsx_r$
that comes from specifying that
$\bsx_j=\bsx_{t,j}$ for $j\in u$, while leaving $\bsx_j = \bsx_{r,j}$ for $j\not\in u$.

In their formulation, the total value to be explained is
$$\Delta(1{:}d) = f(\bsx_t)-\e(f(\bsx_r)),$$
the extent to which $f(\bsx_t)$ differs from a hypothetical
average prediction over independent predictors.
The subset $u$ explains
$\Delta(u)$, and from that they derive Shapley value.
They define quantities $I(u)$ via $I(\emptyset)=0$ and
$$
\Delta(u) = \sum_{v\subseteq u}I(u).
$$
They prove that the Shapley value for predictor $j$ is
$$
\phi_j = \sum_{u\subseteq1{:}d,\,j\in u} \frac{I(u)}{|u|}.
$$
We give a proof of this, different from theirs,
making use of the anchored decomposition.
While their $\Delta(u)$ is defined via expectations of independent random
variables, their Shapley value comes via the anchored decomposition, 
applied to those expectations, not an ANOVA.

\section*{Appendix 2 Proof of Lemma 1}
\begin{proof}
The inclusion-exclusion formula for the binary anchored
decomposition, using $\bsz\in\{0,1\}^d$ is
$$
g_{u,\bszero}(\bsz) = \sum_{v\subseteq u}(-1)^{|u-v|}g(\bsz_v{:}\bszero_{-v}).
$$
Suppose that $z_j=0$ for $j\in u$. Then, splitting up the alternating sum
\begin{align*}
g_{u,\bszero}(\bsz) &= \sum_{v\subseteq u-j}(-1)^{|u-v|}
(g(\bsz_v{:}\bszero_{-v})- g(\bsz_{v+j}{:}\bszero_{-v-j}))\\
&= 0
\end{align*}
because $\bsz_v{:}\bszero_{-v}$
and $\bsz_{v+j}{:}\bszero_{-v-j}$ are the same point when $z_j=0$.
It follows that $g_{u,\bszero}(\bse_w)=0$ if $u\subseteq w$ does not hold.

Now suppose that $u\subseteq w$.
First $g_{u,\bszero}(\bsz)=g_{u,\bszero}(\bsz_u{:}\bsone_{-u})$
because $g_{u,\bszero}$ only depends on $\bsz$ through $\bsz_u$.
From $u\subseteq w$ we have $(\bse_w)_u=\bsone_u$.
Then
$g_{u,\bszero}(\bse_w)=g_{u,\bszero}(\bsone_u{:}\bsone_{-u})=
g_{u,\bszero}(\bsone)$, completing the proof.
\end{proof}

\section*{Appendix 3 Proof of Theorem 1}
\begin{proof}
For $u\ne\emptyset$, $g_{u,\bszero}(\bszero)=0$,
and so the two expressions for $\phi_j$ in~\eqref{eq:shapanchored} are equal.
From the definition of Shapley value,
\begin{align}\label{eq:interpsum}
\phi_j
&= \frac1d\sum_{v\subseteq -j}{ d-1\choose |v|}^{-1}
(g( \bse_{v+j})-g(\bszero))-(g(\bse_v)-g(\bszero))\notag\\
 &= \frac1d\sum_{v\subseteq -j}{ d-1\choose |v|}^{-1}(g( \bse_{v+j})-g(\bse_v))\\
 &= \frac1d\sum_{v\subseteq -j}{ d-1\choose |v|}^{-1}
\sum_{u\subseteq1{:}d}(g_{u,\bszero}( \bse_{v+j})-g_{u,\bszero}(\bse_v)).\notag
\end{align}
By Lemma~\ref{lem:binanchdecomp},
\begin{align*}
\phi_j
&= \frac1d\sum_{v\subseteq -j}{ d-1\choose |v|}^{-1} \times \\
&\sum_{u\subseteq1{:}d}(g_{u,\bszero}( \bsone)1_{u\subseteq v+j}-g_{u,\bszero}(\bsone)1_{u\subseteq v})\\
&= \frac1d\sum_{u\subseteq1{:}d}
g_u(\bsone)\sum_{v\subseteq -j}{ d-1\choose |v|}^{-1}
(1_{u\subseteq v+j}-1_{u\subseteq v}).
\end{align*}
Now
\begin{align}\label{eq:nonzero}
1_{u\subseteq v+j}-1_{u\subseteq v}
=1_{j\in u}1_{j\not\in v}1_{v\supseteq u-j}.
\end{align}
The cardinality of $v$ for which~\eqref{eq:nonzero}
is nonzero ranges from $|u|-1$ to $d-1$ and so
\begin{align*}
\phi_j
&= \frac1d\sum_{u,\, j\in u}g_{u,\bszero}(\bsone) \times \\
&\sum_{r=|u|-1}^{d-1} { d-1\choose r}^{-1}\sum_{v\subseteq -j}
1_{j\in u}1_{j\not\in v}1_{v\supseteq u-j}1_{|v|=r}\\
&= \frac1d\sum_{u,\, j\in u}g_{u,\bszero}(\bsone)
\sum_{r=|u|-1}^{d-1}
{ d-1\choose r}^{-1}{ d-|u| \choose r-|u|+1},
\end{align*}
because $v$ contains $u-j$ and $r-|u|+1$
additional indices from $-u$.
Simplifying
\begin{align*}
{ d-1\choose r}^{-1}{ d-|u| \choose r-|u|+1}
&=
{ r \choose |u|-1}
{ d-1 \choose |u|-1}^{-1}
\end{align*}
and
$$
\sum_{r=|u|-1}^{d-1}{r\choose |u|-1} = {d \choose |u|}
$$
by the ``hockey-stick identity''.
Therefore
\begin{align*}
\phi_j
&= \frac1d\!\sum_{u,\, j\in u}g_{u,\bszero}(\bsone) \!
{d \choose |u|}\!{d-1 \choose |u|-1}^{-1}\!
= \!\sum_{u,\, j\in u}\frac{g_{u,\bszero}(\bsone)}{|u|}.\qedhere
\end{align*}
\end{proof}

The proof in \citet{stru:kono:2010} proceeds
by substituting the inclusion-exclusion
identity into the first expression for $\phi_j$ in~\eqref{eq:shapanchored}
and then showing that it is equal to the
definition of Shapley value.
They also need to explain some of their steps in prose
and the version above provides a more `mechanical' alternative approach.

\section*{Appendix 4 Comparison of Shapley Values}
Here we compare local and global versions of several Shapley value measures for the Titanic and Boston housing data sets.
We include baseline, all baseline and cohort Shapley values as well as their squared versions.
We also calculate variance Shapley values and squared local Shapley values aggregated over all subjects for global sensitivity analysis.

\subsection*{Titanic Data}
Figure~\ref{fig:titan_shapleys} shows the comparison of local Shapley values including cohort Shapley (CS), baseline Shapley (BS), and all baseline Shapley (ABS) and their squared versions.
Each figure shows the Shapley value for each predictor stacked vertically for individual subjects arranged along the horizontal axis. The subjects are ordered by their predicted values.
The black overlay is the sum of predictor contributions for each individual subject.
The blue overlay represents the sum of cohort Shapley values for comparison.

For computation of Shapley values on all 1045 data subjects using a desktop PC with an Intel Core i7-8700 CPU, it spent $6.27s$ for CS, $21.58s$ for BS, and $32.7s$ for ABS, respectively.

Comparing CS (a-1) and squared CS (a-2), the sign of each predictor is flipped whenever the total impact in CS is negative.
The negative importance in squared CS is counter-intuitive but it indicates the predictor negatively affects the prediction outcome for each individual level the same as in the non-squared version of CS.

BS (b-1) and squared BS (b-2) sum to different values than CS (a-1) and squared CS (a-2), respectively, due to the difference between  $f(\bsx_b)$ and $\bar y$.
%
Comparing BS and CS, CS attributes significant importance to `fare' for subjects with high and low prediction, but BS does not and the contribution of `sex' is reduced for these subjects.

Squared ABS (b-2) sums to much larger values for each subject compared with cohort Shapley and baseline Shapley while non-squared ABS (c-1) looks nearly equivalent to BS (b-1).

\begin{figure*}[tb]
  \centering
  \includegraphics[width=1.0 \textwidth]{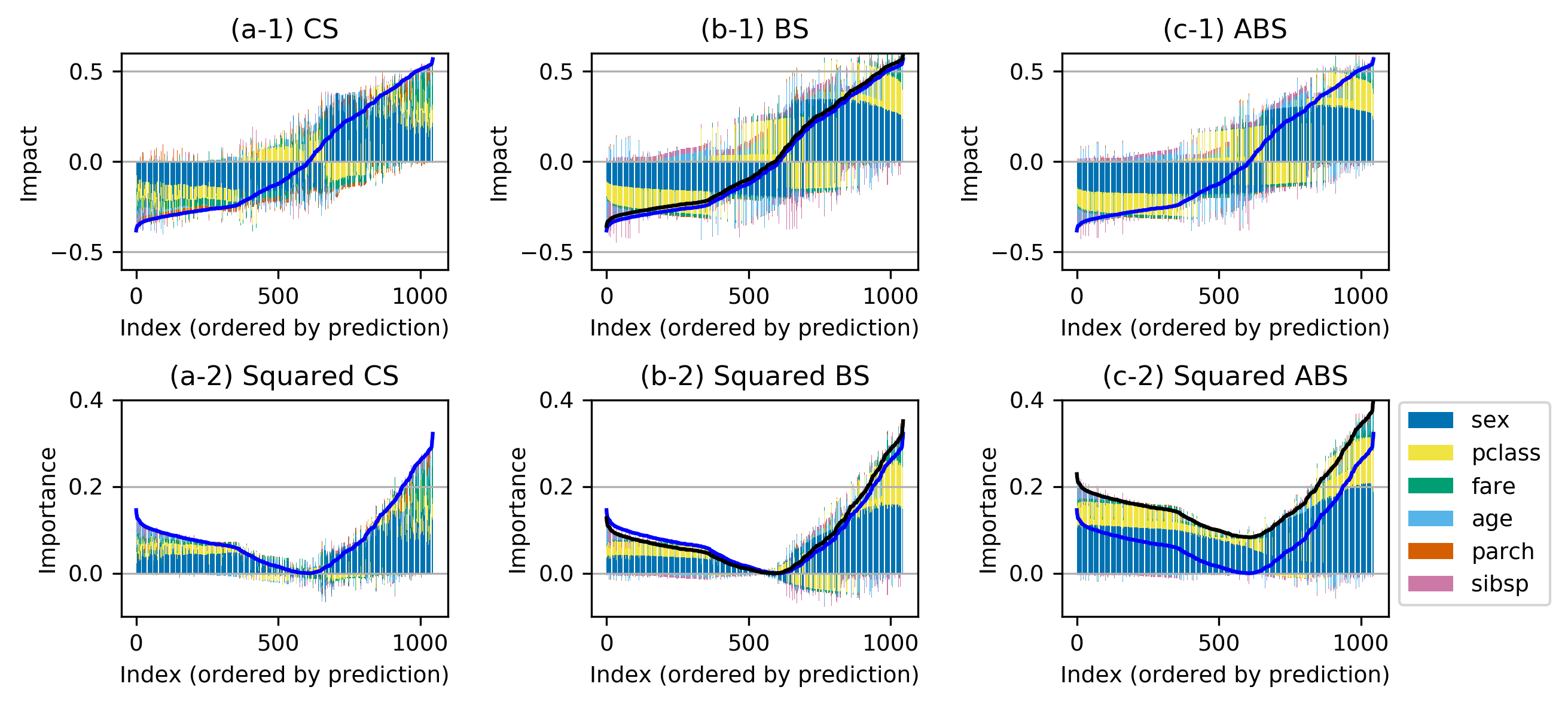}
  \caption{Local Shapley values for the Titanic data. The black curve is their sum.  The blue curve is the sum of cohort Shapley values.}
  \label{fig:titan_shapleys}
\end{figure*}

Figure~\ref{fig:titan_realistic} decomposes the importance measures into their contributions from `realistic' and `unrealistic' differences among synthetic data points.  It is a local comparison for the first point in the data set, i.e., $t=1$.
Our working definition of what makes a feature combination unrealistic is discussed later.
The figure indicates that a substantial proportion of the BS value is derived from unrealistic data values. The baseline point is the average $\bsx_b=(1/n)\sum_{i=1}^{n}\bsx_i$ of predictors, taken over all subjects. Note that some components of the average are already not realistic; for instance the averages of the binary variable `sex' and categorical variable `pclass'.
In addition to unrealistic levels there are also unrealistic combinations.
The ABS by comparison has a much lower impact from unrealistic sampling because it makes comparisons to actual points. Even if the unrealistic contributions are distributed to other predictors, the top two predictors `sex' and `pclass' will not change.
ABS samples only existing values in each predictor.
Through replacing predictor from one subject to another subject, it sometimes goes to unrealistic combination, but that was rare in the Titanic data set.

\begin{figure}[t]
  \centering
  \includegraphics[width=0.8\textwidth]{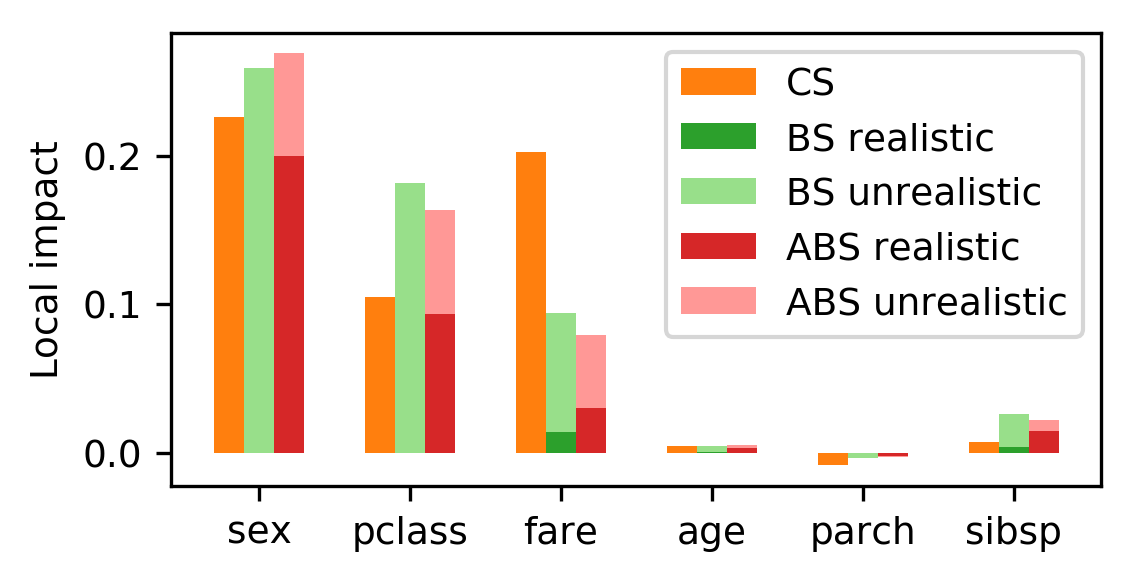}
  \caption{Local Shapley values for subject $1$ of Titanic data set.}
  \label{fig:titan_realistic}
\end{figure}

Figure~\ref{fig:titan_global} shows the comparison of aggregated squared Shapley values.
Squared CS is theoretically equivalent to variance Shapley (VS), though we see small numerical differences. Compared to VS, BS overestimates the impact of `pclass' and `sex', and underestimates that of `sibsp' and `parch'. ABS exaggerates this tendency and the importance estimate becomes nearly double that of CS.

\begin{figure}[t]
  \centering
  \includegraphics[width=0.8 \textwidth]{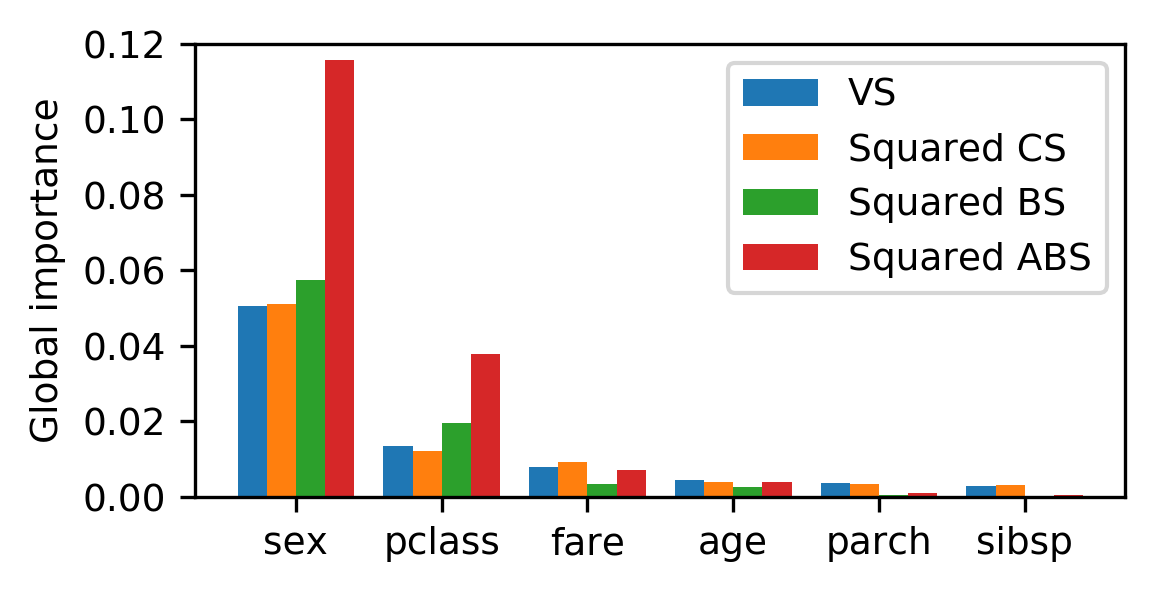}
  \caption{Global Sensitivity Analysis with Shapley values aggregating over all subjects for Titanic data set.}
  \label{fig:titan_global}
\end{figure}

\subsection*{Boston Housing Data}
%
%
%

Figure~\ref{fig:boston_shapleys} shows the comparison of local Shapley values including cohort Shapley (CS), baseline Shapley (BS) and their squared versions.
Each figure shows the Shapley value for each predictor stacked vertically for every individual in horizontal axis. The individuals are ordered by their prediction.
The black overlay is the sum of predictor contributions for each individual subject.
The blue overlay represents the sum of cohort Shapley values for comparison.

For computation of Shapley values on all 506 data subjects using a desktop PC with an Intel Core i7-8700 CPU, it spent $1410.6s$ for CS, $9489.9s$ for BS, and $20994.8s$ for ABS, respectively.

The variance Shapley later shown in Figure~\ref{fig:boston_global} is decomposed into squared CS for individual subjects in Figure~\ref{fig:boston_shapleys}(a-2).
Comparing BS (b-1) and CS (a-1), BS attributes most importance to `STAT', `RM', 'CRIM' and `NOX' while CS attributes importance to more predictors.

\begin{figure*}[tb]
  \centering
  \includegraphics[width=1.0\textwidth]{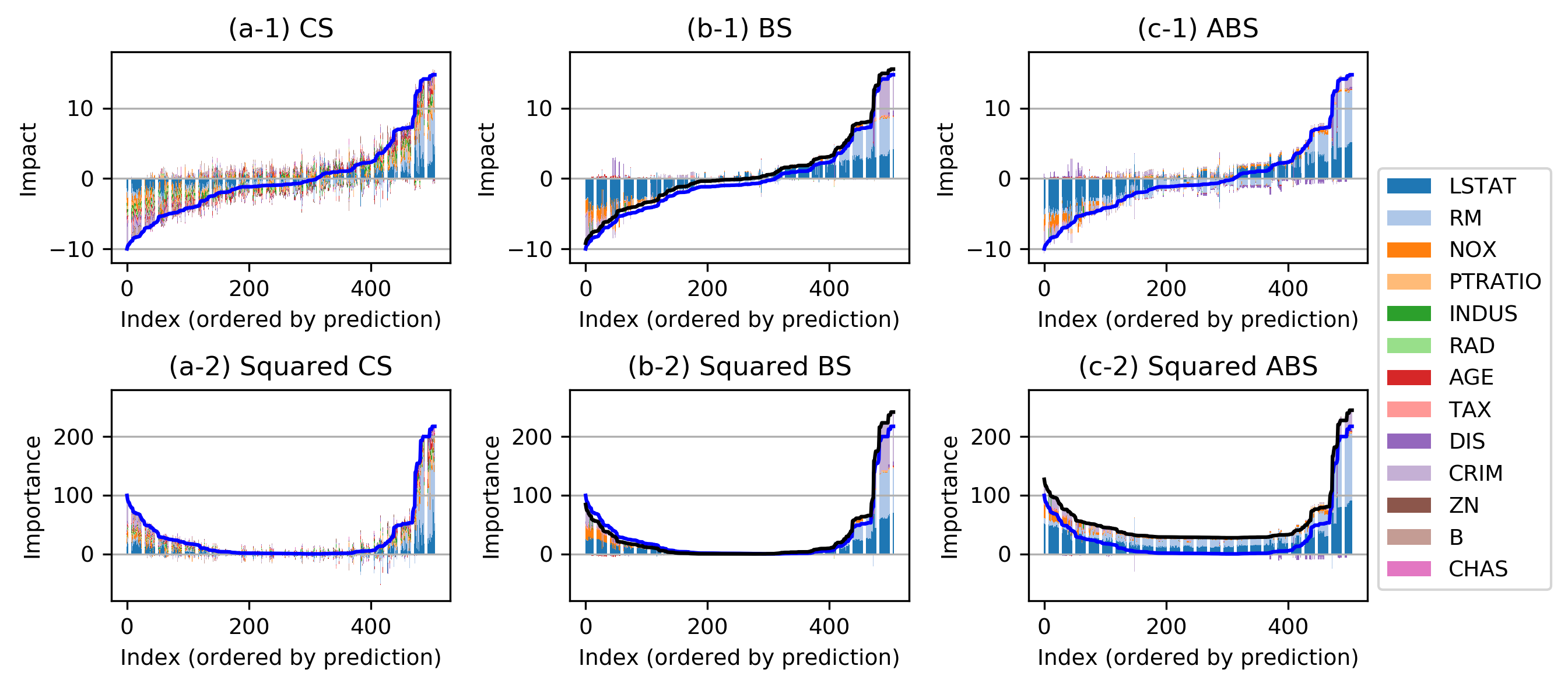}
  \caption{Local Shapley values for Boston housing data set.}
  \label{fig:boston_shapleys}
\end{figure*}

Figure~\ref{fig:boston_realistic} decomposes the local importance measures for subject $t=205$ into their contributions from `realistic' and `unrealistic' differences among synthetic data points. 
The figure indicates that a substantial proportion of the BS value is derived from unrealistic data values. The baseline point is the average $\bsx_b=\sum_{i=1}^{n}\bsx_i$ of predictors, taken over all subjects. In the Boston housing data, all predictors are continuous so that unrealistic samples are all due to unrealistic combinations.

\begin{figure}[t]
  \centering
  \includegraphics[width=0.8\textwidth]{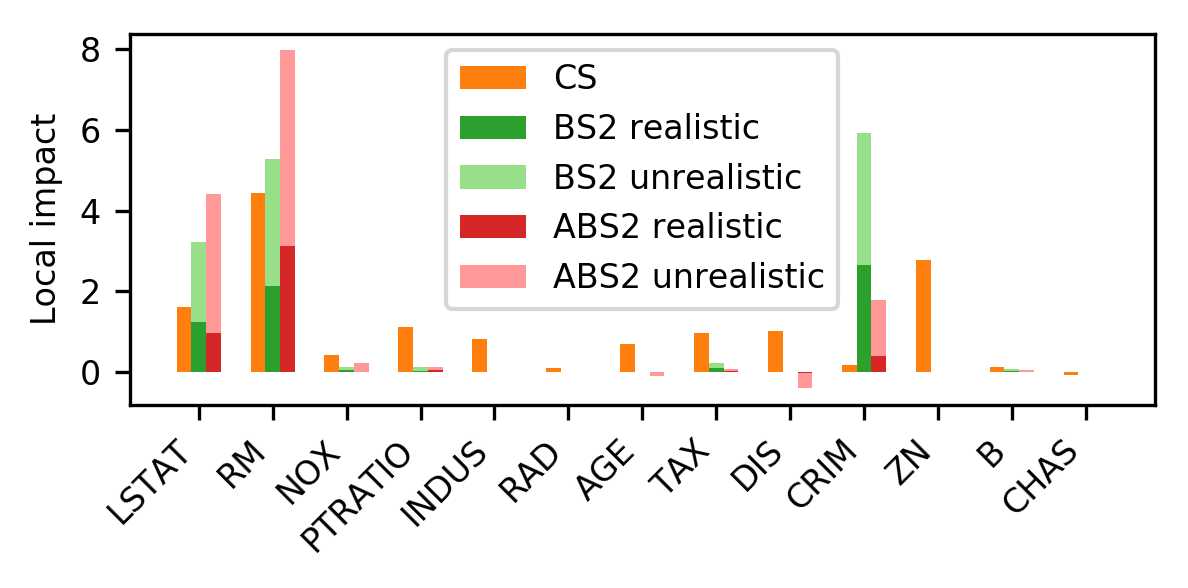}
  \caption{Local Shapley values for subject $205$ of Boston housing data set.}
  \label{fig:boston_realistic}
\end{figure}

Figure~\ref{fig:boston_global} shows the comparison of aggregated squared Shapley values.
As we saw, squared CS is equivalent to variance Shapley (VS).
But BS places much greater
importance on `LSTAT', `RM', and `CRIM' that is observed in visual inspection of Figure~\ref{fig:boston_shapleys}(b-1)(b-2).

\begin{figure}[t]
  \centering
  \includegraphics[width=0.8 \textwidth]{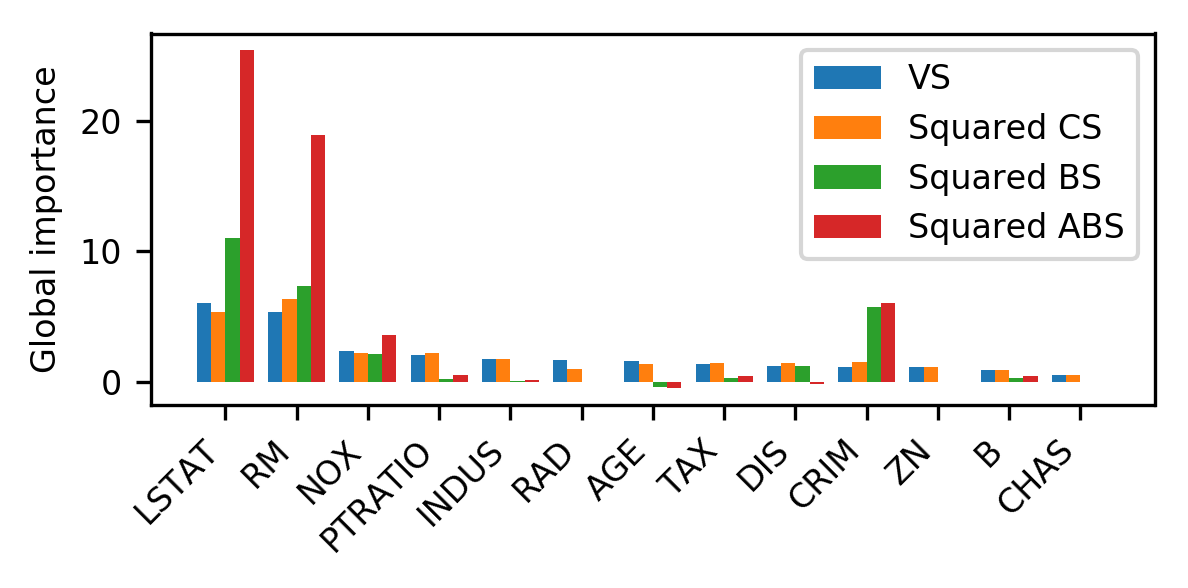}
  \caption{Global Sensitivity Analysis with Shapley values aggregating over all subjects for Boston housing data set.}
  \label{fig:boston_global}
\end{figure}

\section*{Appendix 5 Calibration of Distance Threshold in Realistic Sampling Analysis}
Baseline Shapley occasionally includes unrealistic input data where the model output is questionable.
This arises because the underlying distribution is the product of empirical marginal distributions.
Here, we study the prevalence of unrealistic data samples.

We generate data samples with independent components using the joint marginal distributions in a real world data set.
A sample point is ``realistic'' if there is a subject $i$ to which it is similar in all $d$ predictors. Otherwise it is  ``unrealistic".
We used the similarity measure defined in the Similarity section in this paper. 
Then, we measure the proportion of ``realistic" data samples to quantify the reliability.
We used the distance based similarity function, varying the threshold
in $[0.05, 1.0]$ to explore the effect of the threshold choice. 

A standard technique in machine learning is to hold out data and investigate how well predictions generalize. This motivates comparing the realism of sampling independently from marginal distributions to sampling held out data.
We compare training to test set percentages using holdouts of 10\%, 20\% and 30\% of the data.

Figure~\ref{fig:titan_similarity} shows the resulting proportions of realistic samples when sampling from joint marginal distribution compared to held out data on the Titanic data set,
averaged across $100$ runs. 
We see that held-out data score much higher realism rates than we get from independent sampling of the marginal distributions.  There are only small differences as the hold-out fraction changes from $10$\% to 30\%.
At distance threshold $0.2$ held-out data are realistic 90 to 96 percent of the time, while marginal sampling generates only 86 percent realistic data.

Figure~\ref{fig:boston_similarity} shows the evaluation results on Boston housing data set,
averaged across $100$ runs. 
Once again, the proportion of realistic values does not depend strongly on the hold-out fraction and they are are over 90\% at the distance threshold of 0.2.
The gap between train-test splitting and marginal distribution sampling is much larger compared with the Titanic data. Only about 13\% of the marginal distribution samples are realistic. This is partly to be expected because there are more predictor variables. Furthermore the strongly patterned bivariate distributions in this data serve to make for unrealistic combinations.

\begin{figure}[t]
  \centering
  \includegraphics[width=0.8 \textwidth]{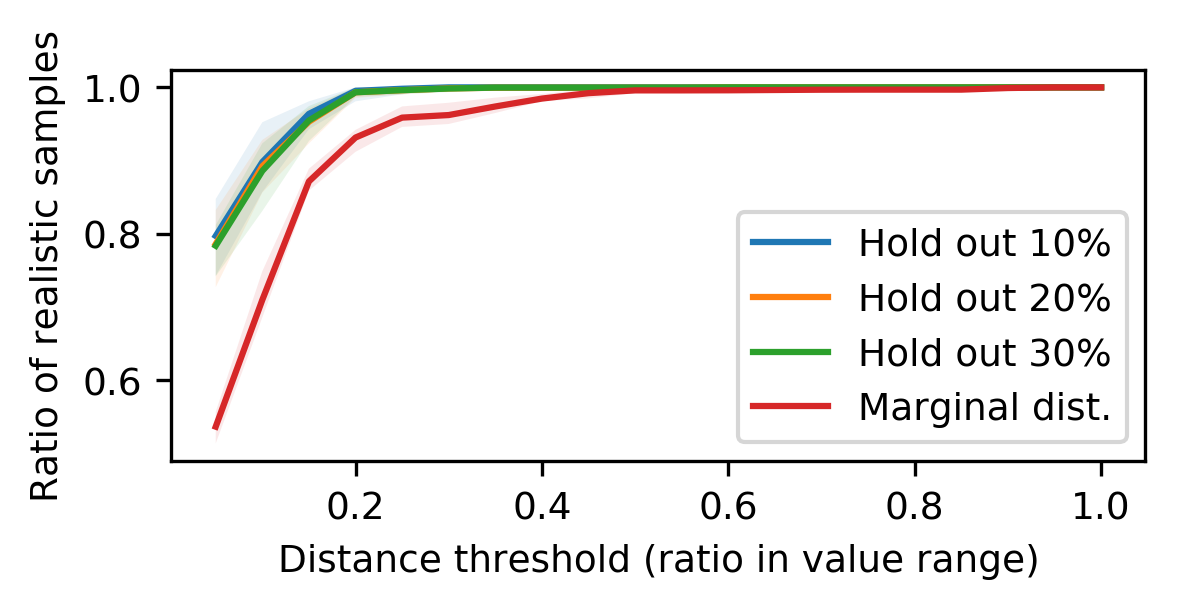}
  \caption{Realistic sampling analysis for Titanic data set.}
  \label{fig:titan_similarity}
\end{figure}

\begin{figure}[t]
  \centering
  \includegraphics[width=0.8 \textwidth]{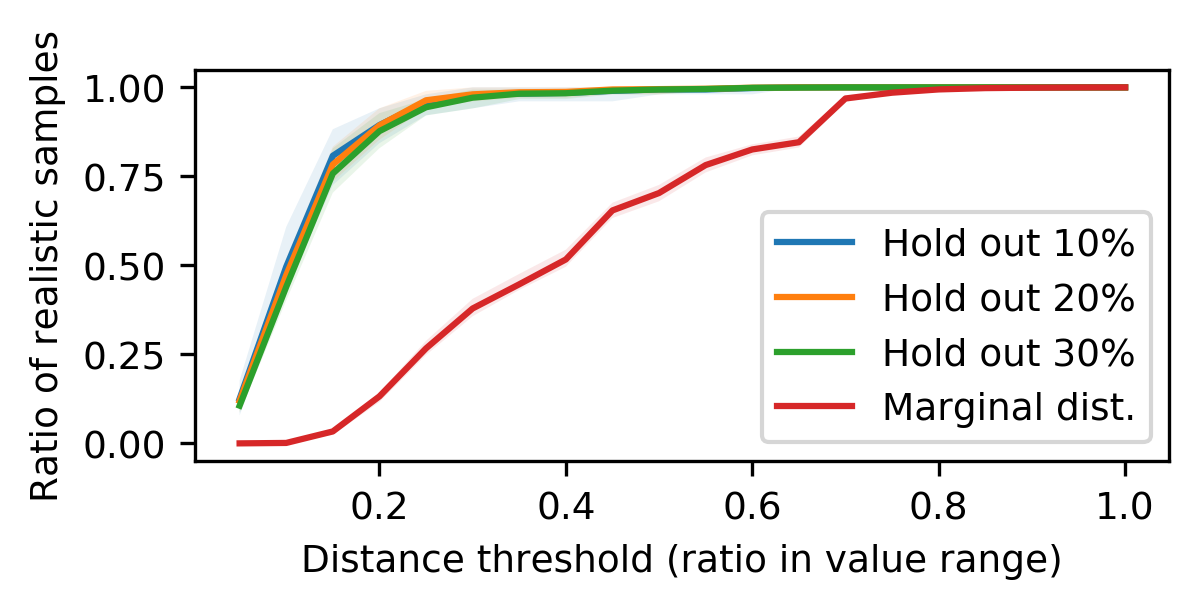}
  \caption{Realistic sampling analysis for Boston housing data set.}
  \label{fig:boston_similarity}
\end{figure}

\end{document}